\documentclass{article}

\PassOptionsToPackage{dvipsnames}{xcolor}

\usepackage[preprint]{neurips_2026}

\usepackage[utf8]{inputenc} 
\usepackage[T1]{fontenc}    
\usepackage{hyperref}       
\usepackage{url}            
\usepackage{booktabs}       
\usepackage{amsfonts}       
\usepackage{nicefrac}       
\usepackage{microtype}      
\usepackage{xcolor}         
\usepackage[table]{xcolor}

\usepackage{graphicx}
\usepackage{subcaption}
\usepackage{bm}
\usepackage[most]{tcolorbox}
\usepackage{amsmath}
\usepackage{amssymb}
\usepackage{mathtools}
\usepackage{amsthm}
\usepackage{multirow}
\usepackage{makecell}
\usepackage{pifont}
\usepackage[capitalize,noabbrev]{cleveref}
\usepackage{xspace}
\usepackage{algorithm}
\usepackage{algorithmic}
\usepackage{wrapfig}
\usepackage{etoc}
\usepackage{titletoc}

\theoremstyle{plain}
\newtheorem{theorem}{Theorem}[section]

\newtheorem{lemma}[theorem]{Lemma}
\newtheorem{corollary}[theorem]{Corollary}
\theoremstyle{definition}

\newtheorem{assumption}[theorem]{Assumption}
\theoremstyle{remark}

\def\red#1{\textcolor{red}{#1}}
\definecolor{checkgreen}{RGB}{76,175,80} 

\usepackage{amsmath,amsfonts,bm,eqnarray}
\usepackage{cancel}
\usepackage{amsthm}
\usepackage{tikz}
\usetikzlibrary{tikzmark}

\newtheorem{innercustomgeneric}{\customgenericname}
\providecommand{\customgenericname}{}
\newcommand{\newcustomtheorem}[2]{%
  \newenvironment{#1}[1]
  {%
   \renewcommand\customgenericname{#2}%
   \renewcommand\theinnercustomgeneric{##1}%
   \innercustomgeneric
  }
  {\endinnercustomgeneric}
}

\newcustomtheorem{customThm}{Theorem}
\newcustomtheorem{customLemma}{Lemma}
\newcustomtheorem{customCor}{Corollary}
\newcustomtheorem{customProposition}{Proposition}

\def\Assumptionref#1{Assumption~\ref{#1}}

\def\Tabref#1{Table~\ref{#1}}

\def\Lmmref#1{Lemma~\ref{#1}}

\def\Figref#1{Fig.~\ref{#1}}

\def\appref#1{Appendix~\ref{#1}}
\def\Secref#1{Sec.~\ref{#1}}

\def\eqref#1{equation~\ref{#1}}
\def\Eqref#1{Eqn.~\ref{#1}}

\def\Algref#1{Algorithm~\ref{#1}}

\def\1{\bm{1}}

\def\rvx{{\mathbf{x}}}

\def\vm{{\bm{m}}}

\def\vx{{\bm{x}}}

\DeclareMathAlphabet{\mathsfit}{\encodingdefault}{\sfdefault}{m}{sl}
\SetMathAlphabet{\mathsfit}{bold}{\encodingdefault}{\sfdefault}{bx}{n}

\renewcommand{\frac}{\tfrac}

\renewcommand{\cite}{\citep}

\title{Few-Step Diffusion Language Models \\ via Trajectory Self-Distillation}

\author{%
  Tunyu Zhang \thanks{
  Equal Contribution. 
  $^1$Rutgers University.
  $^2$Red Hat AI Innovation.
  $^3$MIT-IBM Watson AI Lab.
  $^{\text{\textdagger}}$Correspondence to: 
  Tunyu Zhang <ty.zhang@rutgers.edu>,
  Ligong Han <ligong.han@rutgers.edu>,
  Dimitris N. Metaxas <dnm@cs.rutgers.edu>.
  } \space$^{\text{\textdagger} 1}$ \\
  \And
  Xinxi Zhang $^{*1}$ \\
  \And
  Ligong Han $^{\text{\textdagger}\, 2\,3}$ \\
  \And
  Haizhou Shi $^{1}$ \\
  \And
  Xiaoxiao He $^{1}$ \\
  \And
  Zhuowei Li $^{1}$ \\
  \And
  Hao Wang $^{2\,3}$ \\
  \And
  Kai Xu $^{2\,3}$ \\
  \And
  Akash Srivastava $^{2\,3}$ \\
  \And
  Chengzhi Mao $^{1}$\\
  \AND
  Hao Wang $^{1}$ \\
  \And
  Vladimir Pavlovic $^{1}$ \\
  \And
  Dimitris N. Metaxas $^{\text{\textdagger} 1}$ \\
}
\date{}

\begin{document}

\maketitle

\begin{abstract}
Diffusion large language models (DLLMs) have emerged as powerful generative models with the promise of fast text generation through parallel decoding. However, realizing this potential in practice remains challenging: reducing the number of decoding steps, typically causes a substantial degradation in output quality due to token factorization error. To alleviate this, we propose a self-distillation framework that trains a few-step student to match the \emph{generative trajectory} of a full-step teacher. We theoretically and empirically show that trajectory-level supervision mitigates this factorization error, thereby enabling effective few-step decoding. We further incorporate Direct Discriminative Optimization (DDO), a reverse-KL objective that encourages mode-seeking toward the teacher’s modes, yielding stronger performance on challenging reasoning tasks. Across reasoning and code-generation benchmarks, our method substantially narrows the gap between few-step and full-step decoding. The source code is available at \href{https://github.com/Tyrion58/T3D}{https://github.com/Tyrion58/T3D}.
\end{abstract}





\vspace{-2mm}
\section{Introduction}
Inference-time efficiency is a central challenge in large language modeling, especially for real-time and compute-constrained applications~\cite{zhen2025taming, miao2025towards, alizadeh2024llm}. Diffusion large language models (DLLMs) \cite{labs2025mercury, song2025seed, nie2025large, cheng2025sdar,ye2025dream} offer a promising direction by enabling parallel token generation. However, existing DLLMs rely on long decoding chains consisting of many diffusion steps~\cite{sahoo2024simple, schiffsimple, nie2025large, ye2025dream}, which significantly limits their efficiency gains. When decoding is made more aggressive by reducing the number of steps, these models struggle to accurately generate multiple tokens simultaneously~\cite{cheng2025sdar}.

Recent work \cite{yoo2025redi, chen2025dparallel, xu2024energy, qian2026d3llm, kim2025cdlm, zhang2025variational} has sought to accelerate diffusion large language models (DLLMs) and reduce their inference latency. One line of research focuses on system and decoding improvements, such as better decoding strategies \cite{wu2025fast, chen2026dmax} and adapting KV caching \cite{hu2025accelerating, ma2025dkv, liu2025dllm}. Our work targets an orthogonal bottleneck: the model’s internal prediction structure. In masked diffusion models, few-step decoding is fundamentally limited by the \emph{mean-field} (token-factorized) parameterization~\cite{xu2024energy, yoo2025redi, zhang2025variational}. As the number of decoding steps is reduced, this approximation becomes increasingly inaccurate, leading to a growing factorization error between few-step decoding and full-step decoding. As illustrated in \Figref{fig:main}, the error increases as each step is forced to predict more tokens, causing few-step predictions to deviate further from the full-step model and ultimately degrading generation quality. Previous self-distillation methods~\cite{yoo2025redi, chen2025dparallel} rely primarily on endpoint supervision from the teacher. We argue that this underuses the supervision available in the teacher’s full generative trajectory, which contains much richer information about the model’s prediction structure than the endpoint alone.

Motivated by this, we propose \textbf{Trajectory Self-Distillation}, a principled self-distillation framework for effective few-step decoding in MDLMs. Our core idea is to distill a few-step student by matching the \emph{generative trajectory} of the original full-step teacher, rather than supervising only the endpoint. This exposes the student to richer information about the teacher’s prediction structure and allows it to better approximate full-step decoding under a limited step budget. Building on the analysis of ReDi~\cite{yoo2025redi}, we further show theoretically that trajectory-level supervision reduces factorization error across intermediate reverse transitions. Crucially, our analysis also reveals why prior rectified-flow-style self-distillation does not carry over to MDLMs \cite{labs2025mercury, song2025seed, nie2025large, cheng2025sdar,ye2025dream}, the dominant regime of DLLMs: because the masked prior is deterministic, endpoint-based supervision is uninformative for reducing factorization error. In contrast, our method avoids this failure mode by reducing factorization error over the nontrivial decoding intervals that actually govern few-step generation.

To further improve few-step performance on complex reasoning tasks, we replace the standard forward-KL objective with \emph{Direct Discriminative Optimization} (DDO)~\cite{zheng2025direct}, which encourages the student to focus on the teacher’s high-probability modes. Our intuition is that the mode-covering nature of forward KL can produce over-smoothed predictions and weaker trajectory alignment, while reverse-divergence objectives are inherently mode-seeking and thus yield sharper predictions. In addition, we introduce a \emph{path-consistency} regularizer that places greater emphasis on early decoded tokens, which we find particularly helpful for reasoning.


We term our method Self-\textbf{T}rajectory \textbf{D}istillation via \textbf{DD}O (\textbf{\texttt{T3D}}), a simple self-distillation framework for few-step diffusion language modeling. 
We evaluate \textbf{\texttt{T3D}} on reasoning and code-generation benchmarks using both \emph{SDAR}~\cite{cheng2025sdar}, a block-diffusion language model, and \emph{LLaDA}~\cite{nie2025large}, a full-diffusion language model. 
Across a broad range of decoding budgets and model families, \textbf{\texttt{T3D}} improves over prior few-step DLLM methods, with especially clear gains under aggressive decoding budgets. 
Beyond static few-step decoding, \textbf{\texttt{T3D}} also preserves full-step diffusion performance and remains effective under dynamic decoding. 
Together, these results show that trajectory self-distillation provides a practical route toward efficient few-step diffusion language modeling.

\begin{figure}[t]
    \centering
    \includegraphics[width=0.9\linewidth]{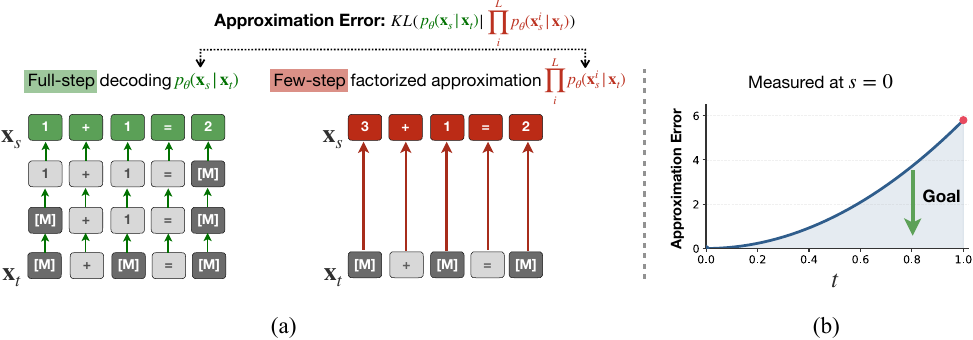}
    \caption{
    \textbf{Factorization error limits few-step decoding in MDLMs.}
    \textbf{(a)} Full-step decoding models the transition $p_\theta(\rvx_s \mid \rvx_t)$ through many denoising steps, while few-step decoding approximate with a token-factorized distribution $\prod_i^L p_\theta(\rvx_s^i \mid \rvx_t)$. This factorized approximation introduces error and can lead to degenerate generations.
    \textbf{(b)} We measured this approximation error on real MATH500 data for $s=0$. The error increases sharply as $t$ grows, showing that larger decoding jumps amplify factorization error. Our goal is to reduce this error and enable reliable few-step decoding.
    }
    \label{fig:main}
\end{figure}

\section{Related Work}

\paragraph{Few-step Diffusion.}
Despite their remarkable success, diffusion models~\cite{yang2023diffusion, ho2020denoising} remain computationally expensive due to their iterative sampling process. 
Consistency Models~\cite{song2023consistency, song2023improved} accelerate generation by enforcing consistency across time, while flow-map-based methods~\citep{geng2025mean, boffi2024flow} reduce sampling cost by directly modeling state-to-state displacements. 
In practice, distillation-based variants often achieve stronger performance, which many attribute to their use of \emph{teacher trajectories}. For example, Consistency Distillation~\cite{song2023consistency} matches teacher intermediate states, CMT~\citep{hu2025cmt} bootstraps training with teacher rollouts, and Re-MeanFlow~\citep{zhang2025flow} leverages teacher-rectified trajectories. Our work aims to bring this trajectory-based perspective to discrete diffusion language modeling.

\paragraph{Efficient Inference for Diffusion Language Models.}
Diffusion large language models (DLLMs) \cite{labs2025mercury, song2025seed, nie2025large, cheng2025sdar,ye2025dream} have recently emerged as powerful generative models for text, but like their continuous counterparts, they require many iterative refinement steps during inference. One line of work improves efficiency through system- and decoding-level advances, such as KV caching~\cite{li2025survey, hu2025accelerating, ma2025dkv}, dynamic decoding~\cite{wu2025fast}, and block-structured diffusion generation~\cite{arriola2025block, cheng2025sdar, wu2025fast2, wang2025diffusion}. Another, orthogonal line of work aims to reduce the number of sampling steps directly. For example, EDLM~\cite{xu2024energy} introduces an energy-based objective to reduce factorization error, dParallel~\cite{chen2025dparallel} distills a few-step model by matching teacher rollouts, and ReDi~\cite{yoo2025redi} adopts a rectified-flow-style~\cite{liu2022flow} distillation procedure. Our work addresses a missing piece in this literature: fully leveraging the supervision available throughout the denoising trajectory.

\section{Background}
\subsection{Masked Diffusion Language Models (MDLMs)}
In this work, we focus on masked diffusion language models (MDLMs)~\citep{sahoo2024simple, shi2024simplified}, as they are the prominent paradigm for current large-scale diffusion language models~\cite{labs2025mercury, song2025seed, nie2025large, cheng2025sdar,ye2025dream}. 

MDLMs are diffusion-based generative models for discrete text sequences.
Let $p_{\mathrm{data}}$ denote the data distribution.
A data sample $\rvx_0 \sim p_{\mathrm{data}}$ is a length-$L$ token sequence
$\rvx_0 = (\rvx_0^{1}, \ldots, \rvx_0^{L})$, where $\rvx_0^{i} \in \mathcal{V}$ denotes a discrete token from a finite vocabulary $\mathcal{V}$ augmented with a special mask token $\vm$.

The \emph{forward (noising)} diffusion process is defined over continuous time $t \in [0,1]$ and corrupts a sequence by independently masking tokens.
The corruption distribution factorizes across tokens:
\begin{align}
\label{eq:forward}
q(\rvx_t \mid \rvx_0)= \prod_{i=1}^{L} q(\rvx_t^{i} \mid \rvx_0^{i}),
\end{align}
where the token-wise kernel $q(\rvx_t^{i} \mid \rvx_0^{i})$ is governed by a
monotonically decreasing noise schedule $\alpha_t \in [0,1]$: at time $t$,
$\rvx_t^{i}$ is preserved as $\rvx_0^{i}$ with probability $\alpha_t$ and replaced by the mask token $\vm$ with probability $1-\alpha_t$. We choose $\alpha_t = 1 - t$ following previous works ~\cite{nie2025large, sahoo2024simple}.

Given a noisier sequence $\rvx_t$, the \emph{reverse (denoising)} process learns to recover a cleaner sequence $\rvx_s$ at an earlier time $s<t$. This reverse transition is approximated by a neural network $p_\theta$ that also factorizes over tokens:
\begin{align}
\label{eq:factorization-error}
    p_\theta(\rvx_s \mid \rvx_t)
    \approx \prod_{i=1}^{L}
    p_\theta\!\left(\rvx_s^{(i)} \mid \rvx_t\right).
\end{align}

As shown in \citep{sahoo2024simple, shi2024simplified}, maximizing the evidence lower bound (ELBO)
for MDLMs admits a remarkably simple form.
Concretely, the optimization objective reduces to a masked-token cross-entropy objective:
\begin{align}
\label{eq:mdm-loss}
\mathcal{L}(\theta)
&=
- 
\mathbb{E}_{\rvx_t \sim q(\cdot \mid \rvx_0)}
\bigl[
\log p_\theta(\rvx_0 \mid \rvx_t)
\bigr]\,
.
\end{align}

\subsection{Direct Discriminative Optimization (DDO)}
Direct Discriminative Optimization (DDO) \citep{zheng2025direct} is a GAN-inspired objective for likelihood-based generative models. Unlike standard GANs \cite{goodfellow2020generative}, which introduce an additional discriminator network, DDO implicitly parameterizes the discriminator using likelihood ratios. Consider a pretrained model $p_{\theta_{\mathrm{ref}}}$ that supplies ``fake'' samples. To distinguish real data $\rvx\sim p_{\mathrm{data}}$ from reference samples $\rvx\sim p_{\theta_{\mathrm{ref}}}$, the optimal discriminator is:
\[
d^*(\rvx)=\frac{p_{\mathrm{data}}(\rvx)}{p_{\mathrm{data}}(\rvx)+p_{\theta_{\mathrm{ref}}}(\rvx)}
      =\sigma\!\left(\log\frac{p_{\mathrm{data}}(\rvx)}{p_{\theta_{\mathrm{ref}}}(\rvx)}\right),
\]
where $\sigma(\cdot)$ denotes the sigmoid function. DDO replaces the unknown $p_{\mathrm{data}}$ by parameterizing a discriminator through a learnable likelihood-based model $p_\theta$:
\[
d_\theta(\rvx) := \sigma\!\left(\log\frac{p_\theta(\rvx)}{p_{\theta_{\mathrm{ref}}}(\rvx)}\right).
\]
Substituting this implicit discriminator into the GAN discriminator loss yields the DDO objective:
\begin{align}
\min_{\theta}\,\mathcal{L}(\theta)
&= -\mathbb{E}_{x\sim p_{\mathrm{data}}}\!\left[\log \sigma\!\left(\log \frac{p_{\theta}(\rvx)}{p_{\theta_{\mathrm{ref}}}(\rvx)}\right)\right] -\mathbb{E}_{x\sim p_{\theta_{\mathrm{ref}}}}\!\left[\log \left(1-\sigma\!\left(\log \frac{p_{\theta}(\rvx)}{p_{\theta_{\mathrm{ref}}}(\rvx)}\right)\right)\right].
\label{eq: ddo}
\end{align}
With unlimited model capacity, \cite{zheng2025direct} show that the global minimizer of the DDO objective above satisfies $p_\theta^* = p_{\mathrm{data}}$. 

\begin{figure}[t]
    \centering
    \includegraphics[width=0.9\linewidth]{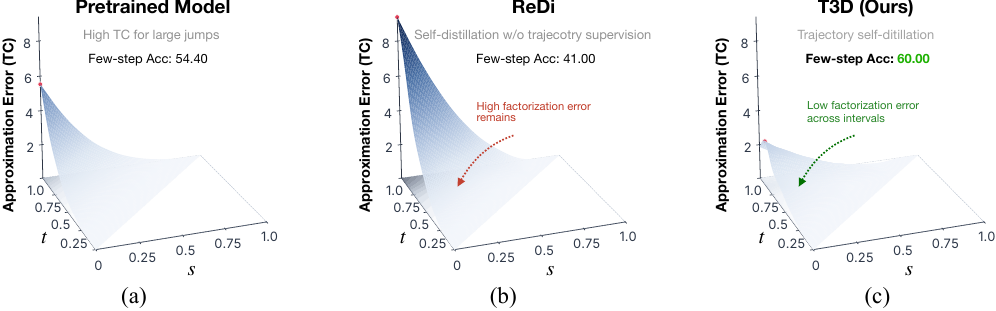}
    \caption{
    \textbf{Approximation Error Surface Across Decoding Interval.}
    We quantify factorization error using conditional total correlation (TC) on MATH500 with SDAR-4B-Chat, visualized over decoding intervals $(s,t)$; larger intervals correspond to more aggressive few-step jumps.
    \textbf{(a)} The pretrained model exhibits high TC for large jumps, leading to degraded few-step accuracy.
    \textbf{(b)} ReDi-style self-distillation without trajectory supervision leaves high TC unresolved, thereby hurting few-step performance.
    \textbf{(c) T3D directly matches the teacher trajectory, substantially lowering TC across intervals and achieving strong few-step accuracy.}
    }
    \label{fig:tc}
\end{figure}
\section{Methods}
\subsection{Factorization Error: The Key Bottleneck in Few-Step MDLMs}

As shown in \Figref{fig:main}, few-step decoding in MDLMs relies on the mean-field parameterization in Eq.~\ref{eq:factorization-error}, which factorizes the reverse transition across tokens. While necessary for tractability, this factorization introduces an approximation error that becomes more severe as the sampling budget is reduced, i.e., as the gap between $s$ and $t$ grows.

Following prior work~\cite{yoo2025redi}, we quantify this error using Conditional Total Correlation (TC), defined as the expected KL divergence between the reverse transition and its token-factorized approximation:
\begin{align}
    TC_J(\rvx_s \mid \rvx_t)
    :=
    \mathbb{E}_{\rvx_t}
    \!\left[
    \mathrm{KL}
    \Big(
     p(\rvx_s \mid \rvx_t)
    \,\big\|\,
    \prod_{i=1}^{L} p(\rvx_s^i \mid \rvx_t)
    \Big)
    \right].
    \label{eq: tc}
\end{align}
Here, the Conditional TC is defined with respect to the joint distribution
$
J(\rvx_s, \rvx_t) = p(\rvx_t)\, p(\rvx_s \mid \rvx_t)
$, induced by the marginal at time $t$ and the reverse posterior.

As shown in Fig.~\ref{fig:main} (b) and \Figref{fig:tc} (a), Conditional TC rises as the decoding interval becomes larger. This identifies factorization error as the central bottleneck in few-step MDLMs: when fewer steps are used, each step must model stronger cross-token dependencies, but the tokenwise factorization becomes increasingly inaccurate, leading to degraded generation quality. 

\subsection{Trajectory Self-Distillation}
\label{sec:traj-distill}
To overcome the factorization bottleneck in few-step MDLMs, we propose \emph{trajectory self-distillation}, which trains a few-step student directly on teacher rollout trajectories.

Specifically, given a pretrained teacher model $p_\phi$, we want to train a few-step student model $p_\theta$ initialized from $p_\phi$ by leveraging pairs of clean and intermediate states $(\rvx_0, \rvx_t)$ sampled along the teacher’s generative trajectory $p_\phi^{\mathrm{Tra}}(\rvx_{0:T})$:
\begin{align}
    p_\phi^{\mathrm{Tra}}(\rvx_{0:T}) = p(\rvx_T)\prod_{t=1}^T p_\phi(\rvx_{t-1} \mid \rvx_t)
\end{align}
Then we define a forward-KL objective that trains the few-step student to match the teacher trajectory, leading to the following self-trajectory distillation loss:
\begin{align}
\label{eq:traj-distill}
\mathcal{L}_{\mathrm{traj}}(\theta)
=
-\,&\mathbb{E}_{p_\phi^{\mathrm{Tra}}(\rvx_t)}
\mathbb{E}_{\rvx_0 \sim p_\phi^{\mathrm{Tra}}(\rvx_0 \mid \rvx_t)} \!\left[\,
\log p_\theta(\rvx_0 \mid \rvx_t)
\right].
\end{align}

Intuitively, this formulation provides substantially richer supervision than endpoint-only distillation ~\cite{chen2025dparallel, yoo2025redi}, since it exposes the student to the teacher’s prediction structure throughout the reverse process rather than only at the final target. More importantly, we show that trajectory self-distillation directly targets the source of few-step failure in MDLMs by reducing approximation error, building on the Conditional TC analysis of ReDi~\cite{yoo2025redi}:
\begin{theorem}[\textbf{Trajectory Distillation Induces Lower Conditional Total Correlation}]
\label{thm:traj-tc}
Let $p_\phi$ be a pretrained teacher model and $p_\theta$ a student model. Define the teacher trajectory joint distribution as
$
J_\phi(\rvx_s,\rvx_t)
$
and the student-induced joint distribution as
$
J_\theta(\rvx_s,\rvx_t).
$
Let $\theta^*$ be the optimal solution to \Eqref{eq:traj-distill}, and let $J_{\theta^*}$ denote the corresponding student joint distribution. Then, for any $s<t$, under mild assumptions, the following inequality holds:
\begin{align}
    \mathbb{E}_t\!\left[TC_{J_{\theta^*}}(\rvx_s \mid \rvx_t)\right]
    \le
    \mathbb{E}_t\!\left[TC_{J_\phi}(\rvx_s \mid \rvx_t)\right].
\end{align}
For proof, please see~\appref{app:proof}.
\end{theorem}

We further show that trajectory-level supervision is particularly important for few-step distillation in MDLMs, where endpoint-only, rectified-flow-style~\cite{liu2022flow} distillation methods such as ReDi~\cite{yoo2025redi} do not directly apply, because they provide no meaningful reduction in factorization error.

\begin{corollary}[\textbf{Endpoint-only Distillation Does Not Reduce Conditional Total Correlation for MDLMs}]
\label{cor:endpoint-vacuous}
In MDLMs \cite{sahoo2024simple, shi2024simplified}, the prior $p(\rvx_T) = \delta_{\mathbf{m}}$ is deterministic. Therefore, for any model,
$
    p(\rvx_0 \mid \rvx_T) = q(\rvx_0),
$
and consequently,
\begin{align}
    TC_J(\rvx_0 \mid \rvx_T)
    =
    \mathrm{KL}\Big(
    q(\rvx_0)
    \,\big\|\,
    \prod_{i=1}^{L} q(\rvx_0^i)
    \Big).
\end{align}
This quantity is a fixed constant of the data distribution and thus cannot be reduced. For proof, please see~\appref{app:proof}.
\end{corollary}

As shown in Fig.\ref{fig:tc} (b), endpoint-only distillation fails to reduce Conditional TC. By contrast, Theorem~\ref{thm:traj-tc} shows that trajectory self-distillation avoids this degeneracy by operating on intermediate states $\rvx_t$. Empirically, this yields a substantially lower TC surface across decoding intervals (\Figref{fig:tc} (c)), which translates into much stronger few-step decoding quality.


\begin{figure}[t]
    \includegraphics[width=\linewidth]{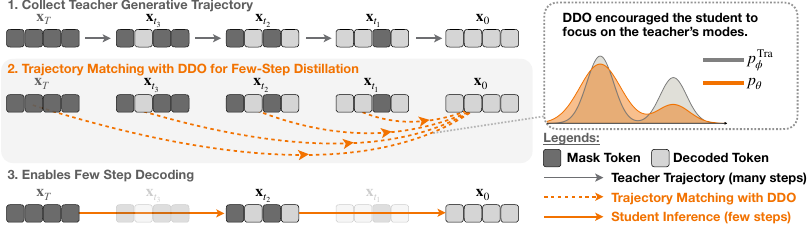}
        \caption{
        \textbf{Overview of T3D.}
        T3D first collects the teacher's full generative trajectory and distills it into a few-step student by matching the student outputs to this teacher trajectory. We use DDO as the trajectory-matching objective, which encourages the student to focus on the teacher's high-probability modes, thereby producing sharper and higher-quality predictions. After distillation, the student can skip intermediate states and perform efficient few-step decoding.
        }
        \label{fig:method}
\end{figure}
\subsection{Improving Trajectory Alignment with DDO}
\label{sec:traj-ddo}
The forward-KL objective in Eq.~\ref{eq:traj-distill} provides a natural way to align the student with teacher trajectories and mitigate the trajectory-level mismatch underlying factorization error.
However, as a mode-covering objective, it can still produce over-smoothed predictions and suboptimal alignment with teacher-generated trajectories. We argue this is harmful for complex reasoning tasks, where sharp decisions on high-probability continuations are often critical. 

Motivated by this, we adopt \emph{Direct Discriminative Optimization} (DDO)~\cite{zheng2025direct} to further improve the few-step quality for trajectory self-distillation.  This GAN-inspired objective induces reverse-KL-like mode-seeking behavior without introducing an additional discriminator. It can be integrated into trajectory self-distillation with minimal modification, while encouraging the student to focus on the teacher's high-probability trajectories.

Formally, we define the trajectory-level DDO objective as:
\begin{align}
\label{eq:total-ddo}
    \mathcal{L}_{\mathrm{traj\text{-}DDO}}(\theta)
    =
    \mathbb{E}_{\rvx_t \sim p_\phi^{\mathrm{Tra}}(\rvx_t)} \bigl[\, l(\theta) \,\bigr],
\end{align}
where the per-step DDO loss is:
\begin{align}
\label{eq:step-ddo}
l(\theta)
&= -\log \sigma\!\left(
\mathbb{E}_{\rvx_0 \sim p_{\phi}^{\mathrm{Tra}}(\rvx_0 \mid \rvx_t)}
\!\left[
\log \frac{p_{\theta}(\rvx_0 \mid \rvx_t)}
{p_{\theta_{\mathrm{ref}}}(\rvx_0 \mid \rvx_t)}
\right]\right)
-\log \!\left(
1 -
\sigma\!\left(
\mathbb{E}_{\rvx_0 \sim p_{\theta_{\mathrm{ref}}}(\rvx_0 \mid \rvx_t)}
\!\left[
\log \frac{p_{\theta}(\rvx_0 \mid \rvx_t)}
{p_{\theta_{\mathrm{ref}}}(\rvx_0 \mid \rvx_t)}
\right]\right)
\right),
\end{align}

where $p_{\theta_{\mathrm{ref}}}$ is a reference model that provides ``fake'' samples and is initialized from $p_\theta$.
The first term encourages the student to assign higher likelihood than the reference model to teacher-generated samples, while the second term penalizes overestimation of samples from the reference model.
As illustrated in Fig.~\ref{fig:entropy}, DDO induces a desirable exploration--exploitation pattern along the decoding trajectory: it maintains higher entropy at the fully masked initial stage, allowing broader exploration, and produces substantially lower entropy afterwards, enabling sharper refinement around teacher-preferred modes.
This sharper trajectory alignment translates into improved few-step reasoning.

\begin{figure}[t]
    \centering

    \begin{minipage}[t]{0.48\linewidth}
        \centering
        \includegraphics[width=0.9\linewidth]{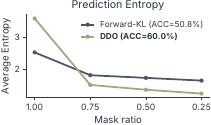}
        \caption{
        \textbf{Average prediction entropy during trajectory distillation.}
        We compare forward-KL and DDO on MATH500 with SDAR-4B-Chat. DDO maintains higher entropy at the fully masked stage but yields substantially lower entropy afterwards, suggesting broad early exploration followed by sharper refinement. This mode-seeking behavior improves trajectory alignment and reasoning accuracy.
        }
        \label{fig:entropy}
    \end{minipage}
    \hfill
    \begin{minipage}[t]{0.50\linewidth}
        \centering
        \includegraphics[width=0.9\linewidth]{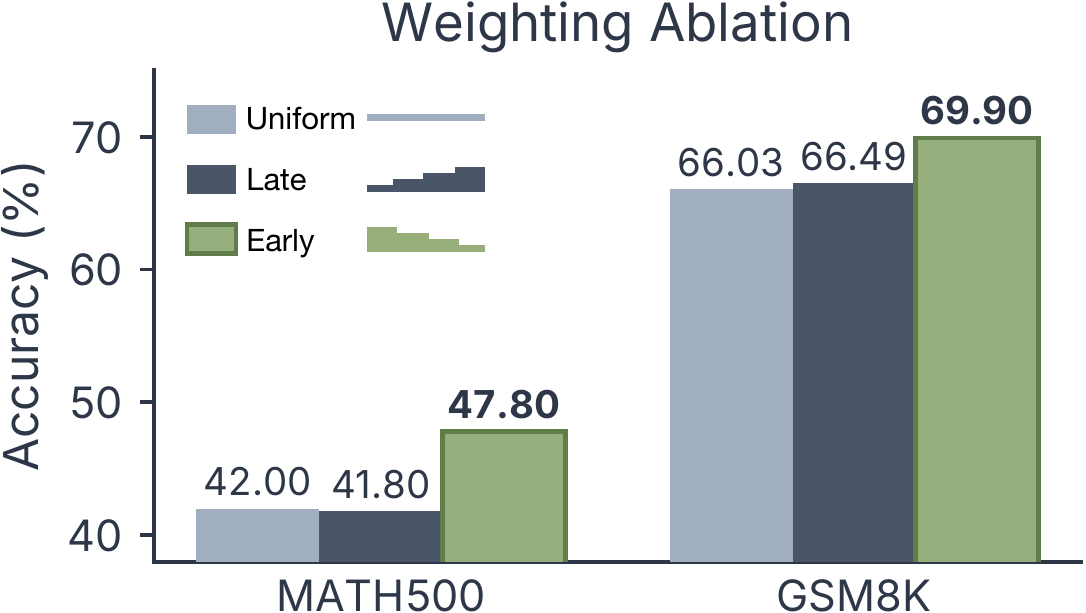}
        \caption{
        \textbf{Comparison of path-consistency weighting strategies for \texttt{T3D}.}
        We compare uniform weighting, late-token weighting, and our early-token weighting on MATH500 and GSM8K. Early-token weighting consistently achieves the best accuracy, suggesting that tokens decoded earlier in the trajectory are more critical under step compression, since their errors can propagate to later decoding steps.
        }
        \label{fig:weight}
    \end{minipage}
    \vspace{-2mm}
\end{figure}

\textbf{Path-Consistency Regularization.}\quad
We further introduce a lightweight \emph{path-consistency} regularization that places larger weight on tokens decoded earlier in the trajectory, since errors at early steps are more likely to propagate under tight decoding budgets. Formally, given a fixed decoding budget $B$, let $\pi_i \in [B]$ denote the decoding step at which token $\rvx_0^i$ is generated, and define the step-dependent weight
$w_i = \frac{B-\pi_i+1}{B}.$  Then we define a token-level weighted path-consistency regularization loss as:
\begin{align}
\label{eq:loss-path}
\mathcal{L}_{\mathrm{path}}(\theta)
= -\,\mathbb{E}_{p_\phi(\rvx_t)}
\mathbb{E}_{\rvx_0 \sim p_\phi(\cdot \mid \rvx_t)}
\!\left[
\sum_i w_i \log p_\theta(\rvx_0^i \mid \rvx_t^{(i)})
\right].
\end{align}
This assigns larger training weight to earlier-decoded tokens while leaving the objective otherwise unchanged. As shown in Fig.~\ref{fig:weight}, we compare against uniform weighting ($w_i=1$) and a late-token schedule ($w_i=\pi_i/B$), which assigns larger weights to later-decoded tokens. Our early-token weighting consistently performs best, indicating that early decoding decisions are more critical under tight step budgets because their errors are more likely to propagate through the remaining trajectory.

\textbf{Final objective.}\quad
Our full method, \textbf{\texttt{T3D}} (\textbf{T}rajectory self-\textbf{D}istillation via \textbf{DD}O), first collects teacher-generated trajectories and then trains a few-step student using DDO together with path-consistency regularization. The final training objective is
\begin{align}
\label{eq:ddo-loss}
    \mathcal{L}_\textbf{\texttt{T3D}}(\theta)
    =
    \mathcal{L}_{\mathrm{traj\text{-}DDO}}(\theta)
    +
    \lambda\, \mathcal{L}_{\mathrm{path}}(\theta),
\end{align}
where $\lambda$ controls the strength of the path-consistency regularization. Figure~\ref{fig:method} provides an overview of the framework, and the full training algorithm is given in \appref{app:algo}.

\vspace{-3mm}
\section{Experiments}
\label{sec:experiments}
\subsection{Experimental Settings}
\label{sec:settings}
\begin{table*}[t]
\caption{
\textbf{Few-step accuracy comparison across baselines on SDAR-1.7B-Chat and SDAR-4B-Chat~\cite{cheng2025sdar}.} 
Few-step performance is evaluated using tokens-per-step (TokPS): for example, \textbf{Block Size}$=4$ and \textbf{TokPS}$=2$ means decoding uses blocks of 4 tokens while generating 2 tokens per diffusion step, resulting in $4/2=2$ diffusion steps per block. \textbf{SD} means Self-Distillation methods. 
\textbf{T3D is consistently among the strongest methods}, demonstrating the effectiveness of trajectory-level distillation for few-step generation.
}
\centering
\small
\resizebox{\linewidth}{!}{%
\setlength{\tabcolsep}{1pt}
\begin{tabular}{clc cccc cccc cc} 
\toprule
\multirow{2}{*}{\textbf{\makecell{TokPS}}} 
& \multirow{2}{*}{\textbf{\makecell{Method}}} 
& \multirow{2}{*}{\textbf{\makecell{SD}}}

& \multicolumn{4}{c}{\textbf{Block Size = 4}} 
& \multicolumn{4}{c}{\textbf{Block Size = 8}}
& \multirow{2}{*}{\textbf{\makecell{AVG.}}}
& \multirow{2}{*}{\textbf{\makecell{Gains (\%)}}}
\\
\cmidrule(lr){4-7} \cmidrule(lr){8-11}
&
&
& MATH500
& GSM8K
& MBPP
& HumanEval

& MATH500
& GSM8K
& MBPP
& HumanEval

\\

\midrule

\multicolumn{13}{c}{\textbf{SDAR-1.7B-Chat}} \\
\midrule
\multirow{6}{*}{2}
& Original Model
& -
& 39.40
& 63.00
& 30.40
& 32.93

& 33.60
& 55.88
& 27.80
& 37.20

& 40.03
& -
\\
& SFT
& \red{\ding{55}}
& 43.00
& 61.79
& 30.00
& 34.76

& 36.80
& 62.55
& 27.20
& 37.80

& 41.74
& \textcolor{checkgreen}{$\uparrow$ 4.28}
\\
\cmidrule{2-13}
& ReDi
& \textcolor{checkgreen}{\ding{51}}
& 40.60
& 63.99
& 13.20
& 16.46

& 36.40
& 62.17
& 12.80
& 13.41

& 32.38
& \red{$\downarrow$ 19.11}
\\
& dParallel
& \textcolor{checkgreen}{\ding{51}}
& 43.40
& 68.23
& 22.20
& 24.39

& 45.20
& 67.70
& 23.20
& \textbf{26.83}

& 40.14
& \textcolor{checkgreen}{$\uparrow$ 0.29}


\\
\rowcolor{gray!15}
& \textbf{\texttt{T3D}} (Ours)
& \textcolor{checkgreen}{\ding{51}}
& \textbf{47.00}
& \textbf{70.96}
& \textbf{27.20}
& \textbf{30.49}

& \textbf{47.80}
& \textbf{68.84}
& \textbf{26.60}
& 25.61

& \textbf{43.06}
& \textbf{\textcolor{checkgreen}{$\uparrow$ 7.59}}
\\
\midrule
\multirow{6}{*}{4}
&Original Model
& - 

& 5.00
& 13.34
& 10.60
& 12.20

& 4.80
& 12.74
& 10.20
& 10.37

& 9.91
& -
\\
& SFT
& \red{\ding{55}}

& 22.40
& 36.62
& 6.20
& 5.49

& 20.00
& 39.65
& 4.40
& 7.93

& 17.84
& \textcolor{checkgreen}{$\uparrow$ 80.05}
\\
\cmidrule{2-13}

& ReDi
& \textcolor{checkgreen}{\ding{51}}

& 15.00
& 32.45
& 3.40
& 5.49

& 12.80
& 29.72
& 4.00
& 4.88

& 13.47
& \textcolor{checkgreen}{$\uparrow$ 35.95}
\\
& dParallel
& \textcolor{checkgreen}{\ding{51}}
& 22.80
& \textbf{45.26}
& \textbf{10.20}
& 12.20

& \textbf{25.40}
& \textbf{42.91}
& \textbf{10.40}
& 11.59

& \textbf{22.60}
& \textbf{\textcolor{checkgreen}{$\uparrow$ 128.09}}



\\
\rowcolor{gray!15}
& \textbf{\texttt{T3D}} (Ours)
& \textcolor{checkgreen}{\ding{51}}

& \textbf{25.60}
& 42.91
& 9.40
& \textbf{15.24}

& 24.40
& 37.38
& 9.20
& \textbf{14.02}

& 22.27
& \textcolor{checkgreen}{$\uparrow$ 124.79}
\\
\midrule

\multicolumn{13}{c}{\textbf{SDAR-4B-Chat}} \\
\midrule
\multirow{6}{*}{2}
& Original Model
& -
& 54.40
& 78.77
& 34.20
& 49.39

& 49.60
& 72.33
& 33.40
& 46.95

& 52.38
& -
\\
& SFT
& \red{\ding{55}}
& 54.60
& 54.60
& 26.80
& 37.20

& 54.44
& 77.41
& 25.60
& 29.88

& 46.76
& \red{$\downarrow$ 10.73}
\\
\cmidrule{2-13}
& ReDi
& \textcolor{checkgreen}{\ding{51}}
& 41.00
& 73.62
& 20.00
& 21.95

& 23.60
& 71.87
& 19.20
& 23.17

& 36.80
& \red{$\downarrow$ 29.74}
\\
& dParallel
& \textcolor{checkgreen}{\ding{51}}
& 52.60
& 76.57
& 23.80
& 39.63

& 51.20
& 75.97
& 18.20
& 28.66

& 45.83
& \red{$\downarrow$ 12.51}


\\
\rowcolor{gray!15}
& \textbf{\texttt{T3D}} (Ours)
& \textcolor{checkgreen}{\ding{51}}
& \textbf{60.00}
& \textbf{83.85}
& \textbf{38.80}
& \textbf{51.83}

& \textbf{61.60}
& \textbf{81.96}
& \textbf{37.00}
& \textbf{56.10}

& \textbf{58.89}
& \textbf{\textcolor{checkgreen}{$\uparrow$ 12.43}}
\\
\midrule
\multirow{6}{*}{4}
& Original Model
& -
& 13.80
& 41.09
& 14.00
& 18.29

& 16.80
& 41.02
& 10.00
& 16.46

& 21.43
& -
\\
& SFT
& \red{\ding{55}}
& 39.00
& 48.14
& 9.00
& 15.85

& 40.20
& 55.42
& 8.80
& 11.59

& 28.50
& \textcolor{checkgreen}{$\uparrow$ 32.98}
\\
\cmidrule{2-13}
& ReDi
& \textcolor{checkgreen}{\ding{51}}
& 25.40
& 53.30
& 5.00
& 7.32

& 20.20
& 47.84
& 6.80
& 6.71

& 21.57
& \textcolor{checkgreen}{$\uparrow$ 0.65}
\\
& dParallel
& \textcolor{checkgreen}{\ding{51}}
& 34.20
& 45.94
& 13.20
& 20.73

& 40.80
& 53.83
& 9.60
& 20.12

& 29.80
& \textcolor{checkgreen}{$\uparrow$ 39.05}
\\


\rowcolor{gray!15}
& \textbf{\texttt{T3D}} (Ours)
& \textcolor{checkgreen}{\ding{51}}
& \textbf{47.80}
& \textbf{69.90}
& \textbf{22.60}
& \textbf{23.78}

& \textbf{44.80}
& \textbf{63.99}
& \textbf{21.20}
& \textbf{23.17}

& \textbf{39.66}
& \textbf{\textcolor{checkgreen}{$\uparrow$ 85.02}}
\\
\bottomrule
\end{tabular}
}

\label{tab:tokps_main}
\vspace{-2mm}
\end{table*}

\paragraph{Baselines.}
We compare T3D with representative few-step diffusion language model baselines:
\emph{ReDi}~\cite{yoo2025redi}, \emph{dParallel}~\cite{chen2025dparallel}, and
\emph{SFT} on real data as a supervised reference.
For LLaDA experiments, we additionally include \emph{CDLM}~\cite{kim2025cdlm},
which accelerates diffusion language models through system-level and training-based designs.
All training-based baselines and T3D are trained until convergence.

\vspace{-10pt}
\paragraph{Models and Benchmarks.}
We evaluate T3D on both block-diffusion and full-diffusion language models.
For block diffusion, we use \textbf{SDAR-1.7B-Chat} and \textbf{SDAR-4B-Chat}~\cite{cheng2025sdar};
for full diffusion, we use \textbf{LLaDA-8B-Instruct}~\cite{nie2025large}.
We evaluate on four reasoning and code-generation benchmarks:
\textbf{MATH500}~\cite{lightman2023let}, \textbf{GSM8K}~\cite{cobbe2021training}, \textbf{MBPP}~\cite{austin2021program}, and \textbf{HumanEval}~\cite{chen2021evaluating}.
These tasks require multi-step reasoning, making them sensitive to quality degradation under aggressive step compression.

\vspace{-10pt}
\paragraph{Metrics.}
For few-step decoding and full-decoding preservation, we report \textbf{Accuracy}.
For dynamic decoding, we additionally report throughput and averaged tokens per decoding steps.
For LLaDA coding tasks, we report \textbf{Extraction Rate (ER)}, following the
limited executable-solution extraction ability of the base model.

\vspace{-10pt}
\paragraph{Training Data and Implementation.}
For self-distillation methods, we collect teacher-generated trajectories from
the corresponding training sets: MATH~\cite{hendrycks2021measuring} for mathematical reasoning and
PrimeIntellect~\cite{jaghouar2024intellect} for code generation.
Unless otherwise specified, trajectories are generated with static decoding and
low-confidence remasking. 
During T3D training, the DDO reference model is periodically updated from the
current student, and we mix random tokens into training inputs to improve robustness.
All trainable methods are fine-tuned using full-parameter training on
$8\times$ NVIDIA A100-40GB GPUs.
More implementation details, including trajectory construction, decoding settings, and training cost, are provided
in~\appref{app:implement}.

\begin{table*}[t]
\caption{
\textbf{Few-step accuracy comparison on LLaDA.}
Following the same protocol as Table~\ref{tab:tokps_main}, we compare T3D with existing few-step decoding and self-distillation baselines under different TokPS settings. \textbf{T3D achieves the best average accuracy} at both TokPS \(=4\) and TokPS \(=8\).
}
\centering
\small
\resizebox{0.8\linewidth}{!}{%
\setlength{\tabcolsep}{4pt}
\begin{tabular}{cl cccc cc} 
\toprule
\textbf{{TokPS}}
& \textbf{{Method}} 
& MATH500
& GSM8K
& MBPP
& HumanEval
& \textbf{{AVG.}}
& \textbf{{Gains (\%)}}

\\
\midrule
\multirow{5}{*}{4}
& Original Model
&24.80
&70.43
&91.80
&87.80

& 68.71
& -
\\
&ReDi
&25.20
&68.39
&93.80
&91.50

& 69.72
& \textcolor{checkgreen}{$\uparrow$ 1.48}
\\
&dParallel
&28.40
&71.49
&94.80
&91.50

& 71.55
& \textcolor{checkgreen}{$\uparrow$ 4.13}
\\
&CDLM
&30.00
&71.70
&85.40
&85.98

& 68.27
& \red{$\downarrow$ 0.64}
\\
\rowcolor{gray!15}
&\textbf{\texttt{T3D}} (Ours)
&\textbf{30.40}
&\textbf{75.89}
&\textbf{98.20}
&\textbf{94.50}

&\textbf{74.75}
&\textcolor{checkgreen}{$\uparrow$ \textbf{8.79}}
\\
\midrule
\multirow{5}{*}{8}
& Original Model
&3.00
&18.04
&40.80
&50.00

& 17.11
& -
\\
&ReDi
&6.80
&31.24
&70.80
&66.50

& 28.71
& \textcolor{checkgreen}{$\uparrow$ 67.80}
\\
&dParallel
&15.60
&53.90
&70.60
&67.70

& 36.58
& \textcolor{checkgreen}{$\uparrow$ 113.76}
\\
&CDLM
&11.60
&46.50
&46.40
&46.95

& 27.43
& \textcolor{checkgreen}{$\uparrow$ 60.29}
\\
\rowcolor{gray!15}
&\textbf{\texttt{T3D}} (Ours)
&\textbf{25.20}
&\textbf{70.13}
&\textbf{86.60}
&\textbf{73.80}

& \textbf{47.13}
& \textcolor{checkgreen}{$\uparrow$ \textbf{175.47}}
\\
\bottomrule
\end{tabular}
}

\label{tab:tokps_main_llada}
\vspace{-4mm}
\end{table*}

\begin{wraptable}{R}{0.52\textwidth}
\vspace{-1.5em}
\noindent\begin{minipage}{\linewidth}
\centering
\small
\caption{\textbf{Preserving diffusion performance under full decoding.}
We revert few-step distilled models to full diffusion decoding using static decoding (one token per step) without additional training.
Results are reported under block size 4 and 4 steps per block, showing that \textbf{\texttt{T3D}} preserves diffusion performance. \textbf{Bold} numbers denote the best result among self-distillation methods.
}
\resizebox{\linewidth}{!}{%
\setlength{\tabcolsep}{3pt}
\vspace{-5mm}
\begin{tabular}{lcccc}
\toprule
Method & MATH500 & GSM8K & MBPP & HumanEval\\
\midrule
\multicolumn{5}{c}{\textbf{SDAR-1.7B-Chat}} \\
\midrule
Original Model
& 59.40
& 80.59
& 45.20
& 59.76
\\
SFT
& 52.00
& 73.09
& 44.20
& 60.37
\\
\midrule
ReDi
& 47.00
& 73.77
& 27.60
& 31.10
\\
dParallel           
& 0.40
& 0.23
& 34.60
& 43.29
\\
TD
& 49.80
& 72.40
& 35.20
& 32.93
\\
\rowcolor{gray!15}
\textbf{\textbf{\texttt{T3D}} (Ours)}
& \textbf{56.80}
& \textbf{78.01}
& \textbf{41.20}
& \textbf{57.32}
\\
\midrule
\multicolumn{5}{c}{\textbf{SDAR-4B-Chat}} \\
\midrule
Original Model
& 68.00
& 89.84
& 58.60
& 71.95
\\
SFT
& 60.20
& 86.05
& 50.20
& 69.51
\\
\midrule
ReDi
& 50.40
& 82.03
& 34.00
& 37.80
\\
dParallel
& 13.20
& 2.88
& 34.00
& 48.17
\\
TD
& 57.40
& 82.11
& 37.60
& 43.90
\\
\rowcolor{gray!15}
\textbf{\textbf{\texttt{T3D}} (Ours)}
& \textbf{70.00}
& \textbf{89.31}
& \textbf{54.20}
& \textbf{73.78}
\\
\bottomrule
\end{tabular}
}
\vspace{-5mm}
\label{tab:full_decoding}
\end{minipage}
\end{wraptable}
\subsection{Improving Performance of Few-Step Decoding by Self-Distillation}

\paragraph{Settings.}
We evaluate few-step decoding under high \textbf{Tokens Per Step (TokPS)}
settings, where larger TokPS corresponds to more aggressive parallel decoding.
For SDAR, we evaluate two block sizes, $4$ and $8$, with TokPS $=2$ and $4$.
For LLaDA, we set the maximum generation length to $1024$, use block size $32$,
and evaluate TokPS $=4$ and $8$.
These settings cover both moderate and highly compressed decoding regimes.

\vspace{-10pt}
\paragraph{Results.}
\Tabref{tab:tokps_main} and~\Tabref{tab:tokps_main_llada} report few-step
accuracy across SDAR and LLaDA models.
Overall, \textbf{\texttt{T3D}} is consistently among the strongest
self-distillation methods and achieves the best average performance in most
settings.
The gains are especially pronounced under more aggressive decoding budgets,
where competing methods often degrade substantially.
These results show that \textbf{\texttt{T3D}} better preserves generation quality
when the diffusion process is compressed to only a few steps.

\vspace{-3mm}
\subsection{Preserving Diffusion Performance under Full Decoding}
\label{sec:full-decoding}

\textbf{Settings.}\quad 
In this experiment, we investigate whether few-step distillation leads to
\emph{diffusion property forgetting}, i.e., whether a model optimized for
compressed decoding degrades when reverted to the original full diffusion
process. To evaluate this, we take models distilled for few-step generation and directly restore them to full diffusion decoding using static decoding strategy, decoding one token per step without any additional training.

\textbf{Results.}\quad \Tabref{tab:full_decoding} reports the results.
Across both SDAR-1.7B-Chat and SDAR-4B-Chat, our methods preserve strong
performance under full decoding.
In particular, \textbf{\texttt{T3D}} achieves performance nearly identical to the
original pretrained model on all benchmarks, and in some cases slightly
outperforms it.
In contrast, prior baselines such as ReDi and dParallel exhibit substantial degradation.

\textbf{Discussions.}\quad These results indicate that trajectory self-distillation does not overfit
to few-step decoding, but instead preserves the model’s fine-grained
denoising capability.
Overall, \textbf{our approach enables few-step generation without sacrificing full diffusion performance}.
\begin{figure}
    \centering
    \includegraphics[width=\linewidth]{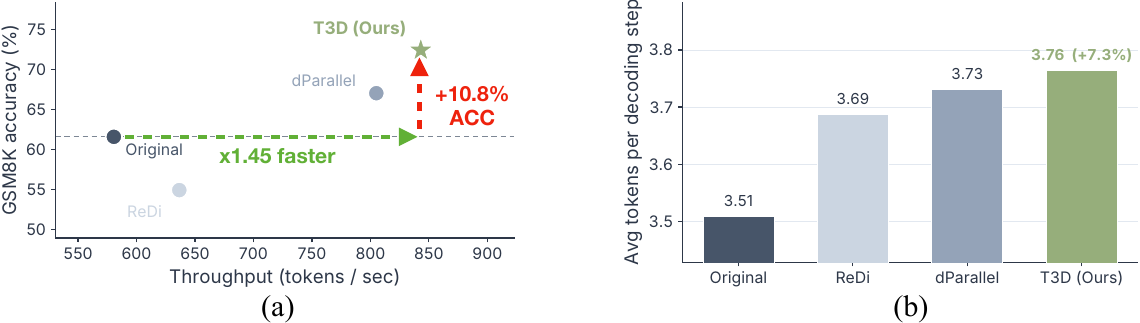}
\caption{
Dynamic decoding results on GSM8K using SDAR-4B-Chat.
(a) Accuracy-throughput trade-off. T3D improves GSM8K accuracy by $+10.8\%$
over the original model while achieving $1.45\times$ higher throughput.
(b) Average decoded tokens per step. T3D decodes more tokens per step under the
same confidence threshold, indicating that it produces more confident predictions
for adaptive decoding. Full results are provided in~\appref{app:dynamic-decoding}.
}
\label{fig:dynamic-decoding}
\end{figure}

\subsection{Experiments on Dynamic Decoding}

\vspace{-10pt}
\paragraph{Settings.}
Dynamic decoding~\cite{wu2025fast, yang2025wavefrontdiffusion} adaptively
determines how many tokens to decode at each step based on model confidence.
Although T3D is trained under fixed static step budgets and our main experiments
use static decoding for controlled comparison, we further evaluate whether the
learned few-step model remains effective when combined with adaptive
decoding strategy.
All dynamic decoding experiments use block size $4$, $4$ steps per block, and a
fixed confidence threshold of $0.9$.

\textbf{Results.}\quad
\Figref{fig:dynamic-decoding} visualizes dynamic decoding on GSM8K, with full
results reported in~\appref{app:dynamic-decoding}.
Under the same dynamic decoding rule, T3D improves the original model by
$+10.8\%$ absolute accuracy while achieving $1.45\times$ higher throughput
(\Figref{fig:dynamic-decoding} a).
It also decodes more tokens per step on average
(\Figref{fig:dynamic-decoding} b), suggesting that T3D produces more confident
predictions and enables larger adaptive decoding steps.
These results show that T3D remains effective beyond the static decoding regime
used during training.



\subsection{Ablation Study}
\label{sec:ablation}
\begin{wraptable}{r}{0.5\linewidth}
\vspace{-1.2em}
\centering
\caption{
Component-wise ablation under aggressive few-step decoding.
Results are averaged over four benchmarks using SDAR-4B-Chat with TokPS $=4$ and block size $=8$.
Setting (c) corresponds to the full T3D objective.
}
\label{tab:ablation-main}
\vspace{-0.5em}
\resizebox{\linewidth}{!}{
\setlength{\tabcolsep}{2pt}
\begin{tabular}{lcc}
\toprule
Method & Acc. & Gains (\%) \\
\midrule
Original & 16.80 & -- \\
(a) \quad + TD ($\mathcal{L}_{\mathrm{traj}}$ in~\Eqref{eq:traj-distill}) & 38.80 & \textcolor{checkgreen}{$\uparrow$ 130.95} \\
(b) \quad + DDO ($\mathcal{L}_{\mathrm{traj\text{-}DDO}}$ in~\Eqref{eq:total-ddo}) & 43.20 & \textcolor{checkgreen}{$\uparrow$ 157.14} \\
\rowcolor{gray!12}
(c) \quad + Path Loss ($\mathcal{L}_{\mathrm{path}}$ in~\Eqref{eq:loss-path}) & \textbf{45.00} & \textbf{\textcolor{checkgreen}{$\uparrow$ 167.86}} \\
\bottomrule
\end{tabular}
}
\vspace{-1.0em}
\end{wraptable}
We conduct a component-wise ablation to examine the contribution of each design in T3D.
As shown in~\Tabref{tab:ablation-main}, trajectory distillation provides the main improvement over the original model, confirming the importance of matching teacher rollout trajectories under aggressive few-step decoding. 
Adding DDO further improves performance, suggesting that mode-seeking trajectory matching produces sharper predictions under tight decoding budgets. 
Finally, the path-consistency loss provides an additional gain by emphasizing early decoded tokens, which helps reduce error propagation. 
Full ablation results across more settings are provided in~\Secref{app:ablation}.

\vspace{-1mm}
\section{Conclusion}
We presented \textbf{\texttt{T3D}}, a simple and effective framework for few-step diffusion language modeling based on \emph{trajectory self-distillation}. Our key insight is that few-step decoding in MDLMs is bottlenecked by factorization error, and that the teacher’s full generative trajectory provides much richer supervision than the endpoint alone for reducing this error. We further uncover a fundamental failure mode of prior endpoint-based, rectified-flow-style self-distillation in MDLMs, and show that \textbf{\texttt{T3D}} avoids this issue by distilling over intermediate decoding intervals, where the reverse process remains informative. Across reasoning and code-generation benchmarks, \textbf{\texttt{T3D}} consistently outperforms prior few-step DLLM methods, substantially narrowing the gap to full-step diffusion decoding.

\textbf{Limitations.}\quad 
Our method has two inherent limitations. First, because it relies on self-distillation, student performance is ultimately bounded by teacher quality. Second, trajectory collection requires full-step teacher rollouts, incurring an offline cost that scales with dataset size and decoding budget.



\bibliographystyle{abbrv}
\bibliography{ref}

\newpage

\appendix

\appendix

\section*{\LARGE Appendix}
\markboth{Appendix}{Appendix}

\startcontents[appendix]

\begingroup
\contentsmargin{0pt}

\titlecontents{section}
  [1.4em]
  {}
  {\makebox[2.0em][l]{\contentslabel{1.6em}}}
  {}
  {\titlerule*[0.5pc]{.}\contentspage}

\titlecontents{subsection}
  [3.2em]
  {}
  {\makebox[3.0em][l]{\contentslabel{2.6em}}}
  {}
  {\titlerule*[0.5pc]{.}\contentspage}

\printcontents[appendix]{}{1}{\setcounter{tocdepth}{2}}
\endgroup
\vspace{1em}

\section{Algorithm}
\label{app:algo}

In this section, we describe the training algorithm of \textbf{\texttt{T3D}}. \Algref{alg:t3d} provides the pseudocode of the full training procedure, while \Figref{fig:method} presents a high-level overview of the method for better conceptual understanding.
\begin{algorithm}[H]
\caption{\textbf{\texttt{T3D}} Training}
\label{alg:t3d}
\begin{algorithmic}[1]

\REQUIRE Teacher model $p_{\phi}$
\REQUIRE Student model $p_{\theta}$ (initialized from teacher)
\REQUIRE Path regularization weight $\lambda$

\STATE Sample trajectory pairs $(\rvx_0, \rvx_t) \sim p_{\phi}$

\REPEAT
    \STATE Set reference model $p_{\theta_{ref}} \leftarrow \text{StopGrad}(p_{\theta})$
    
    \STATE Compute trajectory DDO loss $\mathcal{L}_{\mathrm{traj-DDO}}$
    
    \STATE Compute path consistency loss $\mathcal{L}_{\mathrm{path}}$
    
    \STATE Update student model using
    \[
    \mathcal{L}
    =
    \mathcal{L}_{\mathrm{traj-DDO}}
    +
    \lambda \mathcal{L}_{\mathrm{path}}
    \]

\UNTIL{convergence}

\OUTPUT $p_{\theta}$
\end{algorithmic}
\end{algorithm}

\section{Proof of Theoretical Analysis}
\label{app:proof}
In this section, we provide detailed proofs for the theoretical results presented in the main paper. Our analysis focuses on understanding the behavior of trajectory self-distillation under few-step decoding and its effect on the factorization properties of the reverse diffusion process.

\begin{assumption}
\label{assump:ideal-optimality}
Following~\cite{yoo2025redi}, we assume that the trained student model $p_{\theta^{*}}$ attains the optimum of the MDLM objective:
\begin{align}
\forall\, t\in [0,1], \quad p_{\theta^{*}}=
    \arg\min_{p_{\theta}}\;
    \mathrm{KL}\Big(
   p_{\phi}(\rvx_s \mid \rvx_t)
    \,\big\|\,
    p_\theta(\rvx_s \mid \rvx_t)
    \Big).
\end{align}
\end{assumption}

\begin{assumption}
\label{assump:convex}
Let $P$ be the family of T-step decoding processes. We assume that $\forall t \in [T]$, $p_\theta (\rvx_s|\rvx_t)$ lies within the log-convex hull of $P$. 
\end{assumption}

\begin{lemma}[\textbf{Pythagorean Inequality for KL Divergence}~\cite{wolfer2024geometric}]
\label{lemma:pytha-kl}
Let $\mathcal{Q}$ be a log-convex set. If
$q^* = \arg\min_{q \in \mathcal{Q}} \mathrm{KL}(p \| q)$ and $r \in \mathcal{Q}$,
then
\[
\mathrm{KL}(p \| r)
\ge
\mathrm{KL}(p \| q^*)
+
\mathrm{KL}(q^* \| r).
\]
\end{lemma}

\begin{theorem}[\textbf{Trajectory Distillation Induces Lower Conditional Total Correlation}]
\label{thm:traj-tc-ap}
Let $p_\phi$ be a pretrained teacher model and $p_\theta$ a student model. Define the teacher trajectory joint distribution as
$
J_\phi(\rvx_s,\rvx_t)
$
and the student-induced joint distribution as
$
J_\theta(\rvx_s,\rvx_t).
$
Let $\theta^*$ be the optimal solution to \Eqref{eq:traj-distill}, and let $J_{\theta^*}$ denote the corresponding student joint distribution. Then, for any $s<t$, under mild assumption, the following inequality holds:
\begin{align}
    \mathbb{E}_t\!\left[TC_{J_{\theta^*}}(\rvx_s \mid \rvx_t)\right]
    \le
    \mathbb{E}_t\!\left[TC_{J_\phi}(\rvx_s \mid \rvx_t)\right].
\end{align}
\begin{proof}

\begin{align}
TC_{J_\phi}(\rvx_s\mid \rvx_t)
&=
\mathbb{E}_{\rvx_t\sim p_\phi(\rvx_t)}
\bigg[
\mathrm{KL}(
p_\phi(\rvx_s\mid \rvx_t)
\,\bigg\|\,
\prod_{i=1}^{L} p_\phi(\rvx_s^i \mid \rvx_t)
\bigg)
\bigg]\\
&\ge
\mathbb{E}_{\rvx_t\sim p_\phi(\rvx_t)}
\bigg[
\mathrm{KL}\bigg(
p_\phi(\rvx_s\mid \rvx_t)
\,\bigg\|\,
p_{\theta^*}(\rvx_s \mid \rvx_t)
\bigg)
\nonumber\\
&\qquad\qquad\qquad
+
\mathrm{KL}\bigg(
p_{\theta^*}(\rvx_s\mid \rvx_t)
\,\bigg\|\,
\prod_{i=1}^{L} p_\phi(\rvx_s^i \mid \rvx_t)
\bigg)
\bigg] \\
&\ge \mathbb{E}_{\rvx_t\sim p_\phi(\rvx_t)}
\bigg[\mathrm{KL}\bigg(
p_{\theta^*}(\rvx_s\mid \rvx_t)
\,\bigg\|\,
\prod_{i=1}^{L} p_\phi(\rvx_s^i \mid \rvx_t)
\bigg)
\bigg] \\
&= \mathbb{E}_{\rvx_t\sim p_\phi(\rvx_t)}
\bigg[
\mathrm{KL}\bigg(
p_{\theta^*}(\rvx_s\mid \rvx_t)
\,\bigg\|\,
\prod_{i=1}^{L} p_{\theta^*}(\rvx_s^i \mid \rvx_t)
\bigg)
\nonumber\\
&\qquad\qquad\qquad
+\sum_i^N \mathrm{KL}\bigg(
p_{\theta^*}(\rvx_s^i\mid \rvx_t)
\,\bigg\|\,
\prod_{i=1}^{L} p_\phi(\rvx_s^i \mid \rvx_t)
\bigg)
\bigg] \\
&\geq \mathbb{E}_{\rvx_t\sim p_\phi(\rvx_t)}
\bigg[
\mathrm{KL}\bigg(
p_{\theta^*}(\rvx_s\mid \rvx_t)
\,\bigg\|\,
\prod_{i=1}^{L} p_{\theta^*}(\rvx_s^i \mid \rvx_t)
\bigg)
\bigg] \\
&= TC_{J_{\theta^*}}(\rvx_s\mid \rvx_t).
\end{align}
The first inequality follows from~\Assumptionref{assump:ideal-optimality}, which equates optimizing trajectory self-distillation with minimizing the expected KL divergence. The result then follows by applying \Lmmref{lemma:pytha-kl} and \Assumptionref{assump:convex}. All assumptions invoked here are inherited from~\cite{yoo2025redi}.
\end{proof}
\end{theorem}

\begin{corollary}[\textbf{Endpoint-only Distillation Does Not Reduce Conditional Total Correlation}]
\label{cor:endpoint-vacuous}
In MDLMs, the prior $p(\rvx_T) = \delta_{\mathbf{m}}$ is deterministic. Therefore, for any model,
$
    p(\rvx_0 \mid \rvx_T) = q(\rvx_0),
$
and consequently,
\begin{align}
    TC_J(\rvx_0 \mid \rvx_T)
    =
    \mathrm{KL}\Big(
    q(\rvx_0)
    \,\big\|\,
    \prod_{i=1}^{L} q(\rvx_0^i)
    \Big).
\end{align}
This quantity is a fixed constant of the data distribution and thus cannot be reduced by endpoint-only distillation.
\end{corollary}
\begin{proof}
Since $p(\rvx_T) = \delta_{\mathbf{m}}$, the joint is uniquely $J(\rvx_0, \rvx_T) = q(\rvx_0)\,\delta_{\mathbf{m}}(\rvx_T)$, so $p(\rvx_0 \mid \rvx_T) = q(\rvx_0)$ regardless of the model, and the Conditional TC reduces to the unconditional Total Correlation of the data distribution.
\end{proof}

\section{Implementation Details.}
\label{app:implement}
In this section, we provide implementation details of our method and experimental setup.
\subsection{Mixture of Random Tokens}
\label{app:implement-random}
As described in~\Secref{sec:settings}, we replace some mask tokens with random tokens sampled from the vocabulary $\mathcal{V}$ uniformly. This design is inspired by recent work on one-step discrete generative
modeling for images~\cite{zhu2025di},
where mixing mask tokens with uniformly sampled tokens is shown to improve
training stability and robustness. Formally, let $\rvx = (x_1, \dots, x_L)$ denote a token sequence of length $L$,
and let $\mathcal{V}$ denote the vocabulary.
For each position $i$, we introduce a binary replacement indicator
$r_i \sim \mathrm{Bernoulli}(p_{\mathrm{rand}})$,
where $p_{\mathrm{rand}}$ is the probability of replacing a mask token with
a random token.

\subsection{Multi-Round and Self-Play Update}
\label{app:implement-round}
In our loss function~\Eqref{eq:step-ddo}, we introduce a reference model $p_{\theta_{\text{ref}}}$, which is initialized from the student model $p_\theta$. Following the setup of prior work~\cite{zheng2025direct}, we adopt a multi-round refinement strategy for training. Formally, this process can be written as:
\[
\begin{aligned}
&\text{Round } n:\quad
 \cdots 
\;\rightarrow\;
\underbrace{p_{\theta^{*}_{n-1}}}_{\text{Reference}}
\;\rightarrow\;
\underbrace{
\sigma\!\left(
\beta \log \frac{p_{\theta_n}}{p_{\theta^{*}_{n-1}}}
\right)
}_{\text{Discriminator}}
\\[8pt]
&\text{Round } n+1:\quad
 \rightarrow\;
\underbrace{p_{\theta^{*}_{n}}}_{\text{Reference}}
\;\rightarrow\;
\cdots
\end{aligned}
\]
where $\theta^{*}_{n-1}$ denotes the best-performing student model obtained in round $n$. In each round, the reference model serves as a fixed generator. In our experiments, we update the reference model every 10 global steps, which corresponds to one round in our training schedule.

\subsection{Prompts}
\label{app:implement-prompts}
In this section, we present the prompts used in our experiments. These prompts are used to query the model and generate responses, which are then collected as trajectories for training.
\begin{tcolorbox}[
    breakable,
    enhanced,
    left=-0.5cm, right=-0.5cm, top=2pt, bottom=2pt,
    enlarge top by=0.1cm,
    enlarge bottom by=0.1cm,
    title={\hspace{0.5cm}Prompt For Math Reasoning},
    fonttitle=\bfseries\small
]
\begin{quote}\small
\textbf{[User]}:\quad \texttt{\{problem\}}. Please reason step by step, and put your final answer within \texttt{boxed\{\}}. You are a precise math problem solver. Solve the given math problem step by step. \\
\textbf{[Assistant]}: 
\end{quote}
\end{tcolorbox}

\begin{tcolorbox}[
    breakable,
    enhanced,
    left=-0.5cm, right=-0.5cm, top=2pt, bottom=2pt,
    enlarge top by=0.1cm,
    enlarge bottom by=0.1cm,
    title={\hspace{0.5cm}Prompt For Code Generation},
    fonttitle=\bfseries\small
]
\begin{quote}\small
\textbf{[User]}:\quad This is the problem: \texttt{\{problem\}}. Place your code within a single Python code block \texttt{```python```}. Do not include more than one code block. \\
\textbf{[Assistant]}: 
\end{quote}
\end{tcolorbox}

\subsection{Accelerated Inference}
\label{app:implement-acc}
For all SDAR-series experiments, rollouts are performed using \texttt{JetEngine}~\footnote{https://github.com/Labman42/JetEngine}, a vLLM-style inference framework tailored for diffusion language models. \texttt{JetEngine} is a lightweight yet high-performance inference engine designed for SDAR models and other block-wise diffusion decoding architectures. It supports both dense and MoE models, as well as Tensor Parallel distributed inference, and achieves significant speedups compared to naive inference implementations.

\subsection{Other Implementation Details}
\label{app:other-implement}

\paragraph{Baselines.}
We compare against \emph{ReDi}~\cite{yoo2025redi}, which learns from
teacher-generated clean samples $\vx_0$ paired with randomly corrupted noisy
samples $\vx_t$.
We also include \emph{dParallel}~\cite{chen2025dparallel}, which maximizes the
transition probability from fully masked sequences to teacher-generated clean
sequences.
\emph{SFT} is trained on real data and serves as a supervised reference rather
than a self-distillation baseline.
For LLaDA, we additionally evaluate the official \emph{CDLM}~\cite{kim2025cdlm} checkpoint
under our setting.

\paragraph{Training Data.}
For self-distillation methods, we collect model-generated responses on the
MATH training set~\cite{hendrycks2021measuring} for mathematical reasoning and
the PrimeIntellect dataset~\cite{jaghouar2024intellect} for code generation.
For SFT, we use data derived from Bespoke-Stratos-17k~\cite{labs2025mercury}.
Following prior work~\cite{wang2025diffusion}, we use open-source collections
pre-filtered to a maximum sequence length of 600 tokens.

\paragraph{Trajectory Construction.}
We prompt the teacher model to answer questions from the corresponding training
sets and collect its generated trajectories.
To improve data quality, we use low-confidence remasking
~\cite{nie2025large,wang2025revolutionizing} with static decoding, using block
size $4$ and $4$ steps per block.
To recover the generation trajectory, we record the decoding order of tokens in
the final clean sequence, following prior work~\cite{wang2025revolutionizing}.
Given a clean sequence and its decoding order, intermediate states $\rvx_t$ are
constructed by masking tokens according to the recorded order.
We also mix random tokens into the input for training robustness, following
previous work~\cite{zhu2025di}.

\paragraph{Training Cost.}
All trainable methods are trained with full-parameter fine-tuning on
$8\times$ NVIDIA A100-40GB GPUs.
For SDAR-4B-Chat, trajectory collection takes approximately 1.5 hours with
\texttt{JetEngine} acceleration\footnote{\url{https://github.com/Labman42/JetEngine}},
and T3D training takes approximately 8 hours.
Under the same hardware setting, dParallel and ReDi require roughly 4--5 hours.
Thus, T3D introduces a modest additional offline training cost from DDO, while
the resulting inference speedup applies at deployment time.

\section{Additional Experiments}
\label{app:additional-exp}
This section provides additional experiments that complement the main results. 
We report full dynamic decoding results, analyze full-step diffusion preservation, evaluate robustness across multiple seeds, and examine the generalization of T3D to open-ended language tasks.
\begin{table*}[t]
\centering
\caption{Dynamic decoding results with block size $4$, $4$ steps per block, confidence threshold $0.9$, and temperature $0.1$. We report throughput (TPS), per-sample latency (Latency), average decoding steps and sequence length (Avg Steps and Avg Length), and accuracy (Acc). \textbf{Bold} numbers indicate the best performance among baseline methods. All experiments are done using SDAR-4B-Chat model.}
\label{tab:dynamic-decoding}
\small
\resizebox{\linewidth}{!}{
\setlength{\tabcolsep}{10pt}
\begin{tabular}{llcccccc}
\toprule
Dataset & Method & TPS$\uparrow$ & Latency$\downarrow$ & Avg Steps$\downarrow$ & Avg Length & Acc$\uparrow$ \\
\midrule
\multirow{5}{*}{MATH500}
& Original   
& 657.72 
& 1.10 
& 196.19 
& 721.90 
& 39.00 
\\
\cmidrule{2-7}
& ReDi       
& 715.71 
& 1.04 
& 198.24 
& 757.05
& 27.00 
\\
& dParallel  
& 692.08 
& 0.95 
& 170.22 
& 653.98 
& 45.80
\\
& FKL         
& 693.85 
& 0.97
& 177.99 
& 678.55 
& 44.00 
\\
& \textbf{\texttt{T3D}} (Ours)
& \textbf{791.23} 
& \textbf{0.66} 
& \textbf{137.95} 
& 525.50 
& \textbf{49.40} 
\\
\midrule
\multirow{5}{*}{GSM8K}
& Original   
& 580.60 
& 0.43 
& 71.12 
& 249.52 
& 61.56 
\\
\cmidrule{2-7}
& ReDi       
& 636.58 
& 0.49 
& 84.63 
& 311.99 
& 54.89 
\\
& dParallel  
& 805.02 
& 0.39 
& 83.23 
& 310.58 
& 67.02 
\\
& FKL        
& 696.99 
& 0.47 
& 89.78 
& 330.82 
& 62.40 
\\
& \textbf{\texttt{T3D}} (Ours)        
& \textbf{843.05} 
& \textbf{0.37} 
& \textbf{83.03 }
& 312.48 
& \textbf{72.40} 
\\
\midrule
\multirow{5}{*}{MBPP}
& Original   
& 262.66 
& 0.36 
& 27.25 
& 93.64 
& 23.40 
\\
\cmidrule{2-7}
& ReDi       
& 298.83 
& 0.21 
& 17.11 
& 62.57 
& 10.00 
\\
& dParallel  
& 215.65 
& 0.63 
& 36.03 
& 135.16
& 8.40 
\\
& FKL        
& \textbf{314.99}
& 0.31 
& 26.43 
& 98.80 
& 9.80 
\\
& \textbf{\texttt{T3D}} (Ours)         
& 313.18 
& \textbf{0.19} 
& \textbf{16.94} 
& 61.62 
& \textbf{23.60} 
\\
\midrule
\multirow{5}{*}{HumanEval}
& Original   
& 175.48 
& 0.73 
& 36.56 
& 127.54 
& 33.54 
\\
\cmidrule{2-7}
& ReDi       
& 163.77 
& 0.47 
& 21.23 
& 76.75 
& 10.00 
\\
& dParallel  
& 130.34 
& 0.48 
& 17.41 
& 62.19 
& 23.78 
\\
& FKL         
& 216.39 
& 0.29 
& 17.15 
& 62.10 
& 23.17 
\\
& \textbf{\texttt{T3D}} (Ours)        
& \textbf{222.68} 
& \textbf{0.26} 
& \textbf{16.21} 
& 58.10 
& \textbf{29.27} 
\\
\bottomrule
\end{tabular}
}
\end{table*}

\subsection{Additional Results on Dynamic Decoding}
\label{app:dynamic-decoding}

We provide the full dynamic decoding results in~\Tabref{tab:dynamic-decoding}. 
For reference, we additionally report the Forward-KL variant of T3D, denoted as FKL, which corresponds to the objective in~\Eqref{eq:traj-distill}. 
All experiments use SDAR-4B-Chat with a block size of $4$, $4$ steps per block, a confidence threshold of $0.9$, and a temperature of $0.1$. 
In addition to accuracy, we report throughput, latency, average decoding steps, and output length to characterize the efficiency--quality trade-off.

Overall, T3D consistently achieves strong performance under dynamic decoding.
On MATH500 and GSM8K, T3D improves both accuracy and throughput over the original
model, showing that trajectory self-distillation remains effective even when the
number of decoded tokens is chosen adaptively at inference time. 
On code-generation benchmarks, T3D also substantially improves efficiency and
maintains competitive accuracy. 
These results support the conclusion in the main text that T3D is compatible
with adaptive decoding strategies, although it is trained under static decoding
budgets.

\subsection{Useful Exploration under Reverse-KL Training}
\label{app:reverse}
A potential concern is that the reverse-KL training objective may reduce model diversity, which could harm generation quality on open-ended tasks. However, entropy reduction is not unique to T3D; it is common across many post-training methods, including fine-tuning, RL, and distillation~\cite{gai2025differential, wang2025arbitrary, petrenkoentropy}. More importantly, recent work~\cite{lu2025onpolicydistillation} suggests that lower-entropy, mode-seeking objectives can be beneficial for reasoning, since they concentrate probability mass on coherent solution paths rather than diffuse alternatives.

We provide two additional analyses to examine whether T3D suffers from diversity collapse.

\begin{itemize}
    \item \textbf{Entropy and output diversity.}\quad 
    As shown in~\Figref{fig:dynamic-decoding} (b), T3D does not exhibit uniform diversity collapse relative to Forward-KL. Instead, it shows a stage-wise exploration--exploitation pattern: higher entropy at early decoding stages (mask ratio $=1.0$), indicating broader exploration, and lower entropy at later stages, enabling sharper refinement. This behavior is desirable for reasoning, where the model should explore possible solution paths early and refine toward a coherent answer later.

    \item \textbf{Exploration behavior on reasoning tasks.}\quad 
    As shown in~\Tabref{tab:passk_results}, we compare pass@$k$ on MATH500 for the teacher, Forward-KL baseline, and T3D. T3D outperforms Forward-KL at every $k$. Moreover, the gap between T3D and the teacher narrows as $k$ increases, indicating that T3D preserves meaningful exploration ability rather than collapsing to a narrow set of outputs.
\end{itemize}

\begin{table}[t]
\centering
\caption{
Pass@$k$ results on MATH500 using SDAR-4B-Chat with block size 4 and TokPS 2. 
T3D achieves larger gains as $k$ increases, suggesting that it preserves output diversity and benefits from test-time scaling.
}
\label{tab:passk_results}
\begin{tabular}{lccc}
\toprule
Model & pass@5 & pass@10 & pass@20 \\
\midrule
Teacher & 81.9 & 85.6 & 88.2 \\
Forward-KL & 53.1 & 63.3 & 71.6 \\
T3D (ours) & 66.0 & 74.2 & 80.4 \\
\bottomrule
\end{tabular}
\end{table}

Overall, these results suggest that T3D does not simply reduce diversity in an indiscriminate way. Instead, it preserves useful exploration for reasoning while promoting sharper refinement during decoding. This helps explain why T3D benefits from test-time scaling and consistently improves over Forward-KL under larger pass@$k$ budgets.

\begin{table}[t]
\centering
\caption{
Multi-seed results on MATH-500 and MBPP.
We report accuracy over three seeds, together with the mean, standard deviation, and sample variance.
}
\label{tab:multi-seed}
\resizebox{\linewidth}{!}{
\begin{tabular}{lccccc ccccc}
\toprule
\multirow{2}{*}{Method}
& \multicolumn{5}{c}{MATH-500}
& \multicolumn{5}{c}{MBPP} \\
\cmidrule(lr){2-6} \cmidrule(lr){7-11}
& Seed 1 & Seed 2 & Seed 3 & Mean $\pm$ Std & Var.
& Seed 1 & Seed 2 & Seed 3 & Mean $\pm$ Std & Var. \\
\midrule
Original
& 14.80 & 14.20 & 12.60 & 13.87 $\pm$ 1.14 & 1.29
& 13.00 & 15.00 & 13.40 & 13.80 $\pm$ 1.06 & 1.12 \\
ReDi
& 24.20 & 24.20 & 23.00 & 23.80 $\pm$ 0.69 & 0.48
& 6.40 & 6.80 & 6.60 & 6.60 $\pm$ 0.20 & 0.04 \\
dParallel
& 37.80 & 38.80 & 36.20 & 37.60 $\pm$ 1.31 & 1.72
& 7.00 & 6.20 & 7.00 & 6.73 $\pm$ 0.46 & 0.21 \\
\textbf{T3D (ours)}
& 46.20 & 46.00 & 46.60 & \textbf{46.27 $\pm$ 0.31} & \textbf{0.09}
& 21.60 & 21.60 & 22.40 & \textbf{21.87 $\pm$ 0.46} & 0.21 \\
\bottomrule
\end{tabular}
}
\end{table}

\subsection{Experiments with Multiple Seeds}
To evaluate the stability of T3D, we repeat the main few-step experiments on MATH-500 and MBPP with three random seeds. 
As shown in~\Tabref{tab:multi-seed}, T3D consistently outperforms all baselines across both benchmarks. 
On MATH-500, T3D achieves an average accuracy of $46.27$ with a standard deviation of only $0.31$, indicating both strong performance and low variance across seeds. 
On MBPP, T3D obtains an average accuracy of $21.87$, substantially outperforming the original model and prior few-step DLLM baselines. 
These results suggest that the gains of T3D are stable and not due to seed-specific variation.

\subsection{Generalization to Open-Ended Language Tasks}
\label{app:open-ended}
Our main experiments evaluate T3D on math and coding benchmarks, which test structured reasoning and executable generation under aggressive few-step decoding. To further examine whether T3D remains effective beyond these structured tasks, we additionally evaluate it on WinoGrande, a broader NLP benchmark that requires commonsense language understanding.

As shown in~\Tabref{tab:winogrande_results}, T3D outperforms prior few-step decoding baselines, improving over both dParallel and ReDi. This suggests that the benefits of T3D are not restricted to math or coding tasks, but also extend to more open-ended language tasks.

\begin{table}[t]
\centering
\caption{
Results on WinoGrande. T3D outperforms prior few-step decoding baselines, suggesting that its benefits extend beyond structured math and coding benchmarks.
}
\label{tab:winogrande_results}
\begin{tabular}{lcccc}
\toprule
Model & Original & dParallel & ReDi & T3D (ours) \\
\midrule
Accuracy & 1.0 & 17.4 & 29.7 & \textbf{31.5} \\
\bottomrule
\end{tabular}
\end{table}

Overall, the WinoGrande results provide additional evidence that T3D is not only effective on structured reasoning and coding benchmarks, but can also improve few-step generation on broader open-ended language tasks.

\section{Ablation Study}
\label{app:ablation}
In this section, we present ablation studies for our proposed \textbf{\texttt{T3D}}.
In \appref{app:ablation-lambda}, we analyze the effect of the regularization coefficient $\lambda$.
In \appref{app:full-step-ablation}, we examine how different components of our method contribute to preserving the full diffusion decoding behavior.
Finally, in \appref{app:few-step-ablation}, we present ablations under few-step generation settings to evaluate the contribution of each component to the overall performance of our method.

\subsection{The Effectiveness of $\lambda$ in Training Objective}
\label{app:ablation-lambda}
\textbf{We conduct an ablation study on the regularization weight $\lambda$ in~\Eqref{eq:ddo-loss}}. We run these experiments using the
\textbf{SDAR-4B-Chat} model and evaluate it on \textbf{MATH500} benchmark.
\Tabref{tab:ablation_lambda} reports performance under different decoding
configurations with varying Tokens Per Step (TokPS), block sizes, and decoding
steps.

\textbf{Results.}\quad Overall, moderate regularization consistently yields the best or near-best
performance across most settings.
In particular, $\lambda=0.2$ achieves the strongest results in the majority of
configurations, especially under more aggressive few-step decoding regimes
(e.g., higher TokPS).
In contrast, a smaller regularization weight ($\lambda=0.05$) is often
insufficient to stabilize training, while overly strong regularization
($\lambda=0.5$) can lead to degraded performance in several settings. Based on these observations, we fix $\lambda=0.2$ for all experiments reported
in the main results.
\begin{table}[t]
\centering
\caption{Ablation study on the effect of the regularization weight $\lambda$ under different decoding configurations. We report the model performance across varying Tokens Per Step (TokPS), block sizes, and decoding steps. All experiments are done using MATH500 dataset. }
\label{tab:ablation_lambda}
\begin{tabular}{ccccccc}
\toprule
TokPS & Block Size & Decoding Steps & $\lambda=0.05$ & $\lambda=0.2$ & $\lambda=0.5$ \\
\midrule
1 & 4 & 4 & 67.80 & 69.00 & 69.20 \\
1 & 8 & 8 & 62.60 & 64.80 & 65.40 \\
2 & 8 & 4 & 57.20 & 58.60 & 56.20 \\
4 & 4 & 1 & 47.00 & 47.20 & 46.00 \\
4 & 8 & 2 & 40.20 & 45.20 & 42.00 \\
8 & 8 & 1 & 7.20  & 7.60  & 6.20  \\
\bottomrule
\end{tabular}
\end{table}

\subsection{Preserving Full-Step Diffusion Properties}
\label{app:full-step-ablation}

We first examine whether few-step distillation preserves the original full-step diffusion behavior.
After training each variant under the few-step distillation setting, we revert the model to the original
full-step diffusion decoding process without any additional fine-tuning. This evaluation tests whether
the learned model still retains the fine-grained denoising capability of the pretrained diffusion model.

\begin{table}[t]
\centering
\caption{
Ablation results under full-step diffusion decoding on MATH500.
All variants are trained for few-step distillation and then evaluated by reverting to the original
full-step diffusion process with block size $4$ and decoding steps $4$ per block.
}
\label{tab:appendix-full-step-ablation}
\resizebox{0.8\linewidth}{!}{
\setlength{\tabcolsep}{20pt}
\begin{tabular}{lcc}
\toprule
Method & Objective / Variant & Acc. \\
\midrule
Original & -- & 68.00 \\
SFT & Supervised Fine-Tuning & 60.20 \\
ReDi & Endpoint-Style Distillation & 50.40 \\
\midrule
TD & $\mathcal{L}_{\mathrm{traj}}$ & 22.00 \\
TD + Path Loss & $\mathcal{L}_{\mathrm{traj}} + \lambda \mathcal{L}_{\mathrm{path}}$ & 58.00 \\
\midrule
DDO & $\mathcal{L}_{\mathrm{traj\text{-}DDO}}$ & 12.00 \\
\rowcolor{gray!12}
T3D (Ours) & $\mathcal{L}_{\mathrm{traj\text{-}DDO}} + \lambda \mathcal{L}_{\mathrm{path}}$ & \textbf{69.00} \\
\bottomrule
\end{tabular}
}
\end{table}

\paragraph{Results.}
As shown in~\Tabref{tab:appendix-full-step-ablation}, directly applying few-step distillation can
substantially degrade full-step diffusion behavior. Both TD and DDO alone perform poorly when
the distilled model is reverted to the original full-step decoding process, indicating that optimizing
only for compressed decoding may damage the model's fine-grained denoising capability.

Adding the path-consistency loss substantially improves preservation under full-step decoding.
For TD, adding $\mathcal{L}_{\mathrm{path}}$ improves accuracy from $22.00$ to $58.00$, showing that
path-level supervision helps retain intermediate denoising behavior. The full T3D objective achieves
the best result, reaching $69.00$ accuracy and slightly surpassing the original model. These results
suggest that T3D improves few-step decoding while preserving the intrinsic diffusion behavior of
the pretrained model.

\subsection{Ablation Study on Few-Step Generation}
\label{app:few-step-ablation}

We further study how each component affects few-step generation performance. We evaluate
SDAR-4B-Chat on MATH500 with block size $8$ under two decoding budgets: $4$ decoding steps
per block and $2$ decoding steps per block. The latter corresponds to a more aggressive few-step
decoding regime.

\begin{table}[t]
\centering
\caption{
Component-wise ablation on few-step generation using SDAR-4B-Chat on MATH500.
We evaluate block size $8$ with two decoding budgets: $4$ and $2$ decoding steps per block.
Higher accuracy is better.
}
\label{tab:appendix-few-step-ablation}
\resizebox{0.8\linewidth}{!}{
\setlength{\tabcolsep}{8pt}
\begin{tabular}{lccc}
\toprule
Method & Objective / Variant & 
\begin{tabular}{c} BS $=8$ \\ DS $=4$ \end{tabular} &
\begin{tabular}{c} BS $=8$ \\ DS $=2$ \end{tabular} \\
\midrule
Original & -- & 49.60 & 16.80 \\
SFT & Supervised fine-tuning & 54.44 & 40.20 \\
ReDi & Endpoint-style distillation & 23.60 & 20.20 \\
\midrule
TD & $\mathcal{L}_{\mathrm{traj}}$ & 52.60 & 38.80 \\
TD + Path Loss & $\mathcal{L}_{\mathrm{traj}} + \lambda \mathcal{L}_{\mathrm{path}}$ & 49.40 & 37.20 \\
\midrule
DDO & $\mathcal{L}_{\mathrm{traj\text{-}DDO}}$ & 52.22 & 36.40 \\
\rowcolor{gray!12}
T3D (Ours) & Full objective & \textbf{60.60} & \textbf{45.00} \\
\bottomrule
\end{tabular}
}
\end{table}

\paragraph{Results.}
As shown in~\Tabref{tab:appendix-few-step-ablation}, trajectory-level distillation is the key factor
behind the improvement in few-step generation. TD improves over the original model under both
decoding budgets, especially in the more aggressive setting with only $2$ decoding steps per block,
where accuracy increases from $16.80$ to $38.80$. This supports our main claim that matching
teacher trajectories helps reduce the factorization error induced by large denoising jumps.

DDO further improves few-step performance by replacing the forward-KL trajectory objective with
a mode-seeking trajectory-matching objective. Under the aggressive setting with $2$ decoding steps per block, the full T3D objective performs best,
reaching $45.00$ accuracy. This suggests that the benefit of each component becomes more apparent
as the decoding budget becomes tighter.

Overall, these results show that trajectory supervision provides the main gain, DDO improves the
quality of trajectory matching, and path consistency further stabilizes generation under compressed
decoding.

\newpage

\end{document}